\documentclass[preprint,12pt]{colt2019} 

\usepackage{hyperref}
\usepackage[T1]{fontenc}
\usepackage{microtype}
\usepackage{mathrsfs}
\usepackage{bbm}
\usepackage{algorithm}
\usepackage[noend]{algorithmic}
\usepackage{thmtools,thm-restate}
\usepackage{color}
\usepackage{mathtools}
\usepackage{multirow}
\usepackage{rotating}

\usepackage{enumitem}
\setitemize{noitemsep,topsep=0pt,parsep=0pt,partopsep=0pt}

\usepackage{mathrsfs}
\DeclareMathAlphabet{\mathpzc}{OT1}{pzc}{m}{it}

\newcommand{\abs}[1]{\left| #1 \right|}
\newcommand{\sett}[1]{\left\{ #1 \right\}}
\newcommand{\inner}[2]{\left\langle #1,\, #2 \right\rangle}
\newcommand{\norm}[1]{\left\| #1 \right\|}

\DeclareMathOperator*{\argmin}{arg\,min}

\DeclareMathOperator{\poly}{poly}
\DeclareMathOperator{\polylog}{polylog}

\newcommand{\parens}[1]{\left( #1 \right)}

\newcommand{\vol}[1]{\textrm{Vol}\left( #1 \right)}

\newcommand{\blockremove}[1]{  }

\newcommand{\dec}[0]{\phi}
\newcommand{\enc}[0]{\psi}
\newcommand{\out}[0]{\zeta}

\title[Open Problem: The Oracle Complexity of Convex Optimization with Limited Memory]{Open Problem: The Oracle Complexity of Convex Optimization \\ with Limited Memory}
\usepackage{times}

\coltauthor{\Name{Blake Woodworth} \Email{blake@ttic.edu}\\
 \Name{Nathan Srebro} \Email{nati@ttic.edu}\\
 \addr{6045 S Kenwood Ave, Chicago, IL, 60637}}

\begin{document}

\maketitle
\begin{abstract}We note that known methods achieving the
    optimal oracle complexity for first order convex optimization require
    quadratic memory, and ask whether this is necessary, and more
    broadly seek to characterize the minimax number of first order
    queries required to optimize a convex Lipschitz function subject
    to a memory constraint.
\end{abstract}

\section{Introduction}\label{sec:intro}
We consider first-order optimization methods for convex Lipschitz bounded functions.  I.e., for the following optimization problem
\begin{equation}\label{eq:optimization-problem}
\min_{x\in\mathbb{R}^d:\norm{x}\leq B} F(x)
\end{equation}
where $F$ is convex and $L$-Lipschitz, we consider using (exact)
queries returning $F(x)$ and a subgradient $\nabla F(x)$, and ask the
classical question of how many queries are required to ensure we find
an $\epsilon$-suboptimal solution.  The classical answer is that $O(d
\log LB/\epsilon)$ suffice, using the center-of-mass method, and that,
when $d\ll (LB/\epsilon)^2$ this is optimal \citep{nemirovskyyudin1983}.
However, the center-of-mass method is intractable, at least exactly.
Other methods with a $O(\poly(d) \polylog(LB/\epsilon))$, and even
$\tilde{O}(d \textrm{log}(LB/\epsilon))$, query complexity and
polynomial runtime have been suggested, including as the Ellipsoid
method \citep{shor1970convergence}, Vaidya's method
\citep{atkinson1995cutting} and approximate center of mass using
sampling \citep{bertsimas2002solving}.  But these are generally not
used in practice, since the higher order polynomial runtime dependence
on the dimension is prohibitive.  These methods also all require
storing all returned gradients, or alternatively an ellipsoid in
$R^d$, and so $\Omega(d^2)$ memory.  A simpler alternative is gradient
descent, which requires $O\left( (LB/\epsilon)^2 \right)$ queries, but
only $O(d \log LB/\epsilon)$ memory and $O(d)$ runtime per query.

One might ask: is it possible to achieve the optimal query complexity
using a ``simple'' method?  Since it is much harder to provide runtime
lower bound, we instead focus on the required memory, and ask:
\textbf{is it possible to to achieve the optimal query complexity with
  $O(d\log LB/\epsilon)$ memory?  How does the first-order oracle
  complexity trade off with the memory needed by an optimization
  algorithm?}  This question is formalized in the following Section.

\section{Problem Formulation}
We capture the class of first-order optimization algorithms that use
$M$ bits of memory in terms of a set of ``encoders'' and ``decoders.''
In each iteration, the decoder reads the $M$ bits of memory, and
determines a query point $x_t$.  The encoder receives the function
value $F(x_t)$ and a subgradient $\nabla F(x_t)$ at $x_t$---as is
standard with oracle based optimization, if $F(x_t)$ is not
differentiable at $x_t$, we require the method works for any valid
subgradient used.  The encoder then uses the current memory state,
$F(x_t)$ and $\nabla F(x_t)$ to update the memory state for the next
iteration. At the end, the algorithm's output is chosen as a function
of the final memory state. To be clear, the encoding and decoding
functions can require an arbitrary amount of memory to compute, and
can compute using real numbers. However, between each access to the
oracle, there is a ``bottleneck'' where the algorithm's state must be
compressed down to $M$ bits.

Formally, we define $\mathcal{A}_{M,T}$, the class of all
deterministic\footnote{We focus on deterministic algorithms, but an
  analogous class of randomized algorithms could easily be specified.
  However, it seems unnecessary to complicate things in this way
  because there is evidence that there is little to be gained through
  randomization for solving problems of the form
  \eqref{eq:optimization-problem} \cite{woodworth2017lower}.}
first-order optimization algorithms that use $M$ bits of memory and
$T$ function value and gradient computations. An algorithm $A \in
\mathcal{A}_{M,T}$ is specified by a set of decoder functions
$\sett{\dec_t: \sett{0,1}^M \rightarrow \mathbb{R}^d}_{t=1}^{T}$, a
set of encoder functions $\sett{\enc_t: \sett{0,1}^M \times \mathbb{R}
  \times \mathbb{R}^d \rightarrow \sett{0,1}^M}_{t=1}^{T}$, and an
output function $\out:\sett{0,1}^M\rightarrow\mathbb{R}^d$. The
algorithm's memory is initially blank $\mu_1 = 0$. The $T$ iteration
$t=1...T$ are specified recursively by $x_t = \dec_t(\mu_t)$ and
\begin{equation}
\mu_{t+1} = \enc_t(\mu_t, F(x_t), \nabla F(x_t)) = \enc_t(\mu_t, F(\dec_t(\mu_t)), \nabla F(\dec_t(\mu_t))),
\end{equation}
and the output of the algorithm, denoted $A(F)$, is given by:
\begin{equation}
A(F) = \out(\mu_T).
\end{equation}
Let $\mathcal{F}^d_{L,B}$ be the set of all convex, $L$-Lipschitz
functions $f:\mathbb{R}^d\rightarrow\mathbb{R}$ such that $\exists x^*
\in \argmin_x F(x)$ with $\norm{x^*} \leq B$. We define the minimax
memory-bounded first-order oracle complexity as
\begin{equation}
    T_{L,B}(d,M,\epsilon) = \inf\sett{T \in \mathbb{N}\ :\ \inf_{A \in \mathcal{A}_{M,T}}\sup_{F\in\mathcal{F}^d_{L,B}} F(A(F)) - \min_{\norm{x}\leq B}F(x) \leq \epsilon},
\end{equation}
where the supremum over functions $F$ should be interpreted also as a
supremum over all valid subgradients $\nabla F(x_t)$ used in the
updates.  Without loss of generality, we will fix $L = B = 1$ and
write $T(d,M,\epsilon) := T_{1,1}(d,M,\epsilon)$.  We will further say
that a query-memory tradeoff $(T,M)$ is {\em possible} for a problem
specified by $(d,\epsilon)$ if $T \geq T(d,M,\epsilon)$ and {\em
  impossible} if $T < T(d,M,\epsilon)$.

\begin{figure}\label{fig:memory-oracle-tradeoff-pic}
\centering
\includegraphics[width=10cm]{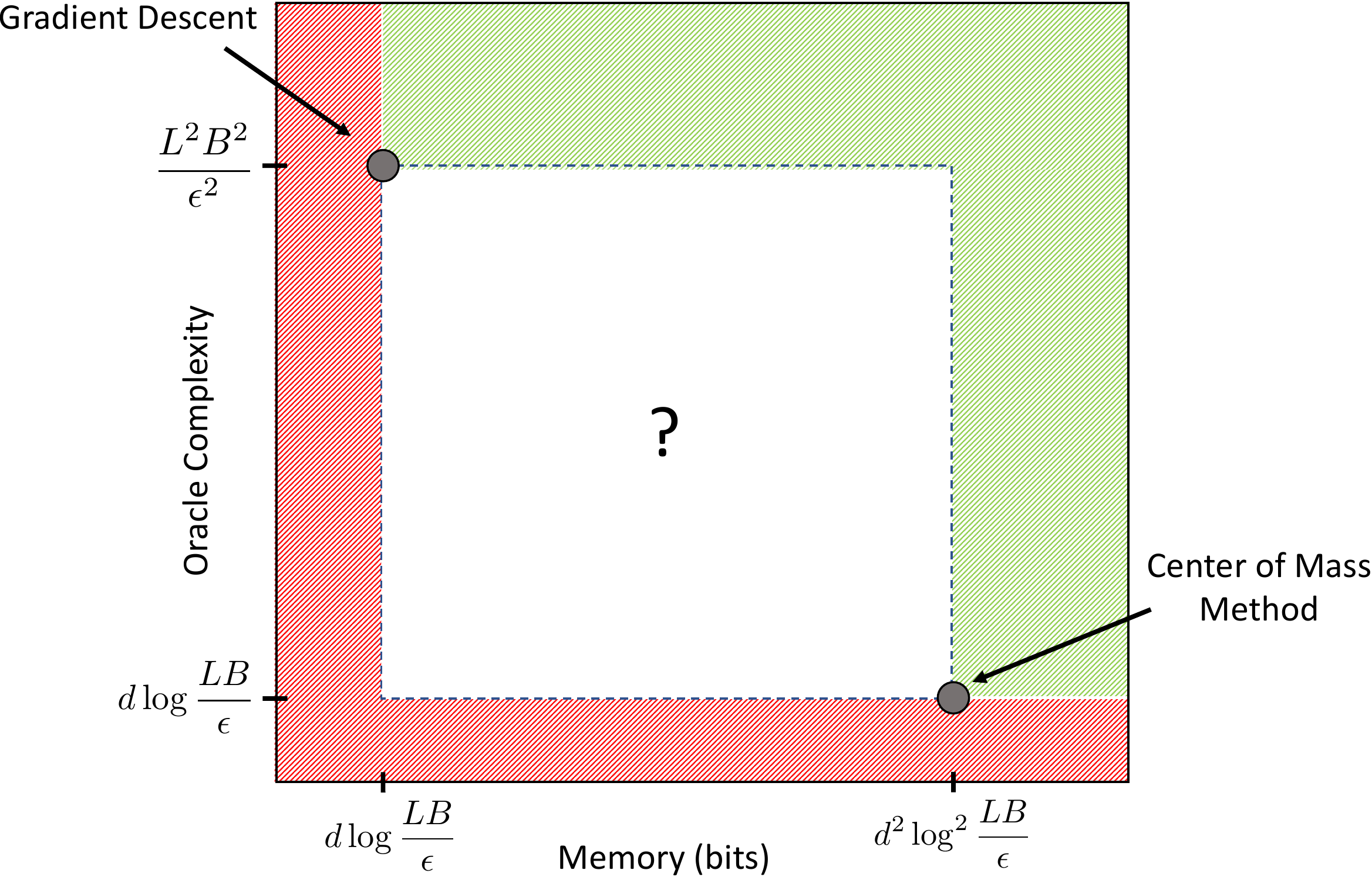}
\caption{The tradeoff between the memory needed and the first-order
  oracle (query) complexity of optimization algorithms.  The shaded
  red ``L'' shaped region along the bottom and left are trade-offs we
  know are impossible.  The shaded green inverted-``L'' shaped region
  along the top and right are trade-offs we know are possible.  We do
  not know whether any trade-offs inside the ``?'' square are possible
  or not.}
\end{figure}

\section{Current Knowledge}

In high dimensions, when $d=\Omega(\tfrac{1}{\epsilon^2 \log
  1/\epsilon})$, gradient descent is optimal in terms of both query
and memory complexity, and so we consider only $d=O(1/\epsilon^2)$.

We can describe the minimax complexity $T(d,M,\epsilon)$, and the query-memory tradeoff, in terms of the regions of possible and impossible $(T,M)$, as depicted in Figure 1.  We currently understand only the extremes.  With any amount of memory, $T=\Omega(d \log 1/\epsilon)$ queries are required, providing a lower bound for the possible region in terms of the query complexity (a horizontal lower bound in Figure 1).  This is attained by the center of mass method, using $O(d^2\log^2(1/\epsilon))$ bits of memory (see Appendix \ref{app:cofmass} for an analysis of Center of Mass with discrete memory), and so any $(T=\Omega(d \log 1/\epsilon),M=\Omega(d^2 \log^2 1/\epsilon))$ is possible (the rectangle above and to the right of ``Center of Mass'' in Figure 1).  At the other extreme, even just representing the answer requires $\Omega(d\log(LB/\epsilon))$ bits (see Theorem \ref{thm:memory-lower-bound} in Appendix \ref{app:lowerbound}), providing a lower bound for the possible region in terms of memory (the vertical lower bound in Figure 1). This is attained by Gradient Descent using $O(1/\epsilon^2)$ queries (see Appendix \ref{app:gd} for an analysis of Gradient Descent with discrete memory), and so any $(T=\Omega(1/\epsilon^2),M=\Omega(d \log 1/\epsilon))$ is possible (the rectangle above and to the left of ``Gradient Descent'' in Figure 1).

To the best of our knowledge, what happens inside the square bordered
by these regions is completely unknown.  Nothing we know would
contradict the existence of a $T,M=O(d \log 1/\epsilon)$ query and
memory complexity algorithm, i.e.~a single optimal method at the
bottom left corner of the unknown square, making the entire square
possible.  It is also entirely possible, as far as we know, that
improving over a query complexity of $\Theta(1/\epsilon^2)$ requires
$\Omega(d^2 \log 1/\epsilon)$ memory, making the entire square
impossible, and implying that no compromise is possible between the
query requirement of Gradient Descent and memory requirement of Center of Mass.

\section{Challenges}

Ultimately, we would like fully understand what is and is not possible:

\paragraph{Question 1 (\$500 or Two Star Michelin Meal)}
Provide a complete characterization of $T(d,M,\epsilon)$ and the
possible $(T,M)$ trade-off, preferably up to constant factors, and at
most up to factors poly-logarithmic in $T$ and $M$.

The most interesting scaling of $d$ and $\epsilon$ is when the dimension is larger then poly-logarithmic but smaller then polynomial in $1/\epsilon$, so that $d \log 1/\epsilon$ memory is less then quadratic memory, but $1/\epsilon^2$ query complexity is not polynomial in $d$.    

\blockremove{
As depicted in Figure 1, by standard first-order oracle complexity lower bounds \cite{nemirovskyyudin1983}, and standard analyses of the center of mass and gradient descent algorithms, we know that for $\epsilon \ll 1/\sqrt{d}$
\begin{align}
    T(d,M,\epsilon) &\geq \Omega\parens{d\log\frac{1}{\epsilon}} \quad \forall M \in \mathbb{N} 
    \\
    T(d,M,\epsilon) &\leq O\parens{d\log\frac{1}{\epsilon}} \quad \forall M \geq \Omega\parens{d^2\log^2\frac{LB}{\epsilon}}
    \\
    T(d,M,\epsilon) &\leq O\parens{\frac{1}{\epsilon^2}} \quad \forall M \geq \Omega\parens{d\log\frac{LB}{\epsilon}}
\end{align}

The question is, how does $T(d,M,\epsilon)$ behave for $\Theta\parens{d\log\frac{LB}{\epsilon}} \leq M \leq \Theta\parens{d^2\log^2\frac{LB}{\epsilon}}$?}

Even without understanding the entire trade-off, it would be interesting to study what can be done on its boundary.  Perhaps the most important regime is the case of linear memory $M = \Theta\parens{d\log\frac{LB}{\epsilon}}$.  Therefore, as a starting point, we ask to characterize $T(d,\Theta\parens{d\polylog \frac{LB}{\epsilon}},1/\epsilon)$.
In particular, is it possible to have query complexity polynomial in $d$ with $\tilde{O}(d)$ memory?

\paragraph{Question 2 (\$200 or One Star Michelin Meal)}
Can we have $T(d,M=\tilde{O}(d),\epsilon)=O(\poly d)$ when $d=\Omega(\log^c 1/\epsilon)$ but $d=O(1/\epsilon^c)$ for all $c$?

At the other extreme, we might ask whether quadratic memory is necessary in order to achieve optimal query complexity:
\paragraph{Question 3 (\$200 or One Star Michelin Meal)}
Can we have $T(d,M=O(d^{2-\delta}),\epsilon)=\tilde{O}(d \polylog 1/\epsilon)$, for $\delta>0$, when $d=\Omega(\log^c 1/\epsilon)$ but $d=O(1/\epsilon^c)$ for all $c$?

The above represent specific incursions into the unknown square in
Figure 1.  Any other such incursion would also be interesting, and provide either for a memory lower bound, or a trade-off improving over Gradient Descent and Center of Mass in some regime.

\paragraph{Question 4 (\$100 or Michelin Bib Gourmand Meal)}
Resolve the possibility or impossibility of some trade-off $(T,M)$ polynomially inside the unknown square in Figure 1.


\bibliographystyle{plain}
\bibliography{bibliography}
\appendix

\section{Analysis of Gradient Descent}\label{app:gd}
\begin{algorithm}
\caption{Memory-Bounded Gradient Descent}
\begin{algorithmic}
\STATE Initialize: $x_0 = 0$, $x_{\textrm{best}} = 0$, $F_{\textrm{best}} = \infty$, $\mu_0 = \sett{x_0, x_{\textrm{best}}, F_{\textrm{best}}}$
\FOR{$t = 0,\dots,T$}
\STATE Compute $F(x_t), \nabla F(x_t)$ \hspace{23mm} $\left.\right\}$ $\phi_t(\mu_t)$
\IF{$F(x_t) < F_{\textrm{best}}$}
\STATE $F_{\textrm{best}} = \textrm{Discretize}(F(x_t))$\hspace*{2cm}%
        \rlap{\smash{\raisebox{-6mm}{$\left.\begin{array}{@{}c@{}}\\{}\\{}\\{}\\{}\\{}\end{array}\right\}%
        \begin{tabular}{l}$\psi_t(\mu_t, \phi(\mu_t))$\end{tabular}$}}}\hspace{3cm}\COMMENT{Discretize s.t.~$F_{\textrm{best}} \leq F(x_t)$}
\STATE $x_{\textrm{best}} = x_t$
\ENDIF
\STATE $\tilde{x}_{t+1} = x_t - \eta_t \nabla F(x_t)$
\STATE $x_{t+1} = \textrm{Discretize}(\tilde{x}_{t+1})$
\STATE $\mu_{t+1} = \sett{x_{t+1}, x_{\textrm{best}}, F_{\textrm{best}}}$
\ENDFOR
\RETURN $x_{\textrm{best}}$
\end{algorithmic}
\end{algorithm}

\begin{theorem}\label{thm:gd}
For any $L$-Lipschitz and convex function $F$ with $\norm{x^*} \leq B$, the gradient descent algorithm can find a point $\hat{x}$ with $F(\hat{x}) - F^* \leq \epsilon$ using $O\parens{d\log\frac{LB}{\epsilon}}$ bits of memory and $O\parens{\frac{L^2B^2}{\epsilon^2}}$ function and gradient evaluations.
\end{theorem}
\begin{proof}


For now, assume that in each iteration the perturbation of the gradient descent iterates resulting from the discretization $\tilde{x}_t \mapsto x_t$ is bounded in L2 norm, i.e. $\norm{\tilde{x}_t - x_t} \leq D$.
Then, following the standard gradient descent analysis, 
\begin{align}
\norm{x_{t+1} - x^*}^2 
&= \norm{\tilde{x}_{t+1} - x^* + x_{t+1} - \tilde{x}_{t+1}}^2 \\
&\leq \norm{x_t - \eta_t\nabla F(x_t) - x^*}^2 + D^2 + 2\inner{\tilde{x}_{t+1} - x^*}{x_{t+1} - \tilde{x}_{t+1}} \\
&= \norm{x_t - x^*}^2 + \eta_t^2\norm{\nabla F(x_t)}^2 - 2\eta_t \inner{\nabla F(x_t)}{x_t - x^*} + D^2 + 2DB \\
&\leq \norm{x_t - x^*}^2 + \eta_t^2L^2 - 2\eta_t\parens{F(x_t) - F^*} + D^2 + 2DB
\end{align}
Rearranging this expression, we conclude
\begin{align}
F(x_t) - F^*
&\leq \frac{1}{2\eta_t}\parens{\norm{x_t - x^*}^2 - \norm{x_{t+1} - x^*}^2} + \frac{\eta_t L^2}{2} + \frac{D^2 + 2DB}{2\eta_t}
\end{align}

Choosing a fixed stepsize $\eta_t = \eta = \frac{B}{L\sqrt{T}}$ and averaging the iterates, we conclude $\bar{x}_T = \frac{1}{T}\sum_{t=1}^T x_t$ achieves suboptimality
\begin{align}
F(\bar{x}_T) - F^* 
&\leq \frac{1}{T}\sum_{t=1}^T F(x_t) - F^* \\
&\leq \frac{1}{T}\sum_{t=1}^T \frac{1}{2\eta}\parens{\norm{x_t - x^*}^2 - \norm{x_{t+1} - x^*}^2} + \frac{\eta L^2}{2} + \frac{D^2 + 2DB}{2\eta} \\
&= \frac{L}{2B\sqrt{T}}\parens{\norm{x_1 - x^*}^2 - \norm{x_{T+1} - x^*}^2} + \frac{LB}{2\sqrt{T}} + \frac{(D^2 + 2DB)L\sqrt{T}}{2B} \\
&\leq \frac{LB}{\sqrt{T}} + \frac{(D^2 + 2DB)L\sqrt{T}}{2B}
\end{align}

Thus, $D \leq \frac{B}{T}$ ensures
\begin{align}
F(\bar{x}_T) - F^* 
&\leq \frac{LB}{\sqrt{T}} + \frac{(D^2 + 2DB)L\sqrt{T}}{2B} \\
&\leq \frac{LB}{\sqrt{T}} + \parens{\frac{B}{2T^2} + \frac{2B}{2T}}L\sqrt{T} \\
&\leq \frac{3LB}{\sqrt{T}}
\end{align}
Since the averaged iterate achieves this suboptimality, the best iterate's suboptimality is at least this good. For $T \geq \frac{9L^2B^2}{\epsilon^2}$, this ensures that at least one of the iterates was $\epsilon$-suboptimal. As long as $F_{\textrm{best}}$ is discretized to accuracy $\epsilon$, then $x_{\textrm{best}}$ is at most $2\epsilon$-suboptimal. This discretization can be achieved using $\log\frac{2LB}{\epsilon}$ bits.

Discretizing the iterates up to accuracy $D = \frac{B}{T} = \frac{\epsilon^2}{9L^2B}$ can be achieved using the log of the $\frac{\epsilon^2}{9L^2B}$ L2 covering number of the radius-$B$ ball, which is upper bounded by $d\log\parens{1 + \frac{36L^2B^2}{\epsilon^2}}$ bits. The discretization of $x_\textrm{best}$ is achieved using the same number of bits.

Therefore, the total number of bits of memory needed to implement gradient descent is at most
\begin{equation}
    2\cdot d\log\parens{1 + \frac{36L^2B^2}{\epsilon^2}} + \log\frac{2LB}{\epsilon} = O\parens{d\log\frac{LB}{\epsilon}}
\end{equation}
\end{proof}

\section{Analysis of Center of Mass Algorithm}\label{app:cofmass}
\begin{lemma}[\cite{grunbaum1960partitions}]\label{lem:grunbaum}
For any convex set $K \subseteq \mathbb{R}^d$ with center of gravity $c$, and any halfspace $H = \sett{x : \inner{a}{x-c} \geq 0}$ passing through $c$, 
\[
    \frac{1}{e} \leq \frac{\vol{K \cap H}}{\vol{K}} \leq 1 - \frac{1}{e}
\]
\end{lemma}

\begin{algorithm}
\caption{Memory-Bounded Center of Mass}
\begin{algorithmic}
\STATE Initialize: $K_0 = \sett{x: \norm{x} \leq B}$, $\mu_0 = \varnothing$
\FOR{$t = 0,\dots,T$}
\FOR{$k = 1,\dots,t$}
\STATE $c_{k-1} = \int_{K_{k-1}} x dx / \int_{K_{k-1}} dx$
\STATE $K_k = K_{k-1} \cap \sett{x : \inner{\tilde{\nabla} F(c_{k-1})}{x - c_{k-1}} \leq 0}$\hspace*{3mm}%
        \rlap{\smash{\raisebox{3.25mm}{$\left.\begin{array}{@{}c@{}}\\{}\\{}\\{}\\{}\end{array}\right\}%
        \begin{tabular}{l}$\phi_t(\mu_t)$\end{tabular}$}}}
\ENDFOR
\STATE $c_t = \int_{K_t} x dx / \int_{K_t} dx$
\STATE $\tilde{F}(c_t), \tilde{\nabla} F(c_t) = \textrm{Discretize}\parens{F(c_t), \nabla F(c_t)}$\hspace*{1.42cm}%
        \rlap{\smash{\raisebox{-3.25mm}{$\left.\begin{array}{@{}c@{}}\\{}\\{}\end{array}\right\}%
        \begin{tabular}{l}$\psi_t(\mu_t, \phi(\mu_t))$\end{tabular}$}}}
\STATE $\mu_{t+1} = \mu_t \cup \sett{\tilde{F}(c_t), \tilde{\nabla} F(c_t)}$
\ENDFOR
\RETURN $\textrm{Discretize}\parens{\argmin_{c \in \sett{c_1,\dots,c_T}} \tilde{F}(c)}$
\end{algorithmic}
\end{algorithm}

\begin{theorem}\label{thm:com}
For any $L$-Lipschitz and convex function $F$ with $\norm{x^*} \leq B$, the center of mass algorithm can find a point $\hat{x}$ with $F(\hat{x}) - F^* \leq \epsilon$ using $O\parens{d^2\log^2\frac{LB}{\epsilon}}$ bits of memory and $O\parens{d\log\frac{LB}{\epsilon}}$ function and gradient evaluations.
\end{theorem}
\begin{proof}
This proof is quite similar to existing analysis of the center of mass algorithm \citep{bubeck2015convex}, we simply take care to count the number of required bits.

Consider the set $K^\alpha = \sett{(1-\alpha) x^* + \alpha x : x \in K_0}$, which has volume $\vol{K^\alpha} = \alpha^d \vol{K_0}$. By convexity, 
\begin{align}
F((1-\alpha) x^* + \alpha x) &\leq (1-\alpha) F^* + \alpha F(x) \\
&\leq (1-\alpha) F^* + \alpha \parens{F^* + \norm{\nabla F(x)}\norm{x - x^*}} \\
&\leq F^* + 2\alpha LB
\end{align} 
By Gr{\"u}nbaum's Lemma, $\vol{K_{T}} \leq \parens{1 - \frac{1}{e}}\vol{K_{T-1}} \leq \parens{1 - \frac{1}{e}}^{T}\vol{K_0}$. Thus, when $T \geq 3d\log(1/\alpha)$
\begin{equation}
    \vol{K_T} \leq \parens{1-\frac{1}{\epsilon}}^T\vol{K_0} < \alpha^d \vol{K_0} = \vol{K^\alpha}
\end{equation}
We conclude that there must be some iteration $t$ in which $\exists y \in K^\alpha \cap \parens{K_{t} \setminus K_{t+1}}$. Thus, $y \in K^\alpha$ and $\inner{\tilde\nabla F(c_t)}{y-c_t} > 0$. We will now argue that the center of mass $c_t$ has small error:
\begin{align}
F(c_t) 
&\leq F(y) + \inner{\nabla F(c_t)}{c_t - y} \\
&\leq F^* + 2\alpha LB + \inner{\nabla F(c_t) - \tilde\nabla F(c_t)}{c_t - y} + \inner{\tilde\nabla F(c_t)}{c_t - y} \\
&\leq F^* + 2\alpha LB + \norm{\nabla F(c_t) - \tilde\nabla F(c_t)}\norm{c_t - y} + 0 \\
&\leq F^* + 2B\parens{\alpha L + \norm{\nabla F(c_t) - \tilde\nabla F(c_t)}}
\end{align}
Therefore, if we choose $\alpha = \frac{\epsilon}{4LB}$ and discretize gradients with L2 error at most $\frac{\epsilon}{4B}$, this ensures that $F(c_t) - F^* \leq \epsilon$. 

The gradients of an $L$-Lipschitz function are contained in the Euclidean ball of radius $L$. Therefore, the gradients can be discretized with error $\frac{\epsilon}{4B}$ using the logarithm of the $\frac{\epsilon}{4B}$ L2 covering number of the Euclidean ball of radius $L$, which is upper bounded by $d \log\parens{1 + \frac{16LB}{\epsilon}}$ bits. There are $T = 3d\log(1/\alpha) = 3d\log\parens{\frac{4LB}{\epsilon}}$ gradients in total, thus the total number of bits required to represent the gradients is
\begin{equation}
    T\cdot d \log\parens{1 + \frac{16LB}{\epsilon}} = 3d^2\log\parens{\frac{4LB}{\epsilon}} \log\parens{1 + \frac{16LB}{\epsilon}}
\end{equation}
Since $\epsilon \leq LB$, this is upper bounded by $3d^2 \log^2\parens{\frac{17LB}{\epsilon}}$ bits. 

Once $T$ iterations have been completed, we know that at least one of the centers must be an $\epsilon$-approximate minimizer of the objective. Using the stored gradients, we can then recompute all centers and return a discretization of the best center. As long as $\abs{\tilde{F}(c_t) - F(c_t)} \leq \epsilon$ for all $t$, then the center chosen by the algorithm will be within $\epsilon$ of the best center. This discretization of the function values requires $T\cdot\log\parens{\frac{2LB}{\epsilon}}$ bits.

As long as the discretization of the chosen center has L2 error at most $\epsilon/L$, then the output will be a $3\epsilon$-approximate minimizer. The number of bits needed for this discretization is at most $d \log\parens{1 + \frac{4LB}{\epsilon}}$. Therefore, the total number of bits needed is at most
\begin{equation}
    3d^2 \log^2\parens{\frac{17LB}{\epsilon}} + 3d\log\parens{\frac{4LB}{\epsilon}}\log\parens{\frac{2LB}{\epsilon}} + d \log\parens{1 + \frac{4LB}{\epsilon}} = O\parens{d^2 \log^2\parens{\frac{LB}{\epsilon}}}
\end{equation}
Rescaling $\epsilon' = \epsilon/3$ completes the proof.
\end{proof}

\section{Memory Lower Bound}\label{app:lowerbound}
\begin{lemma}\label{lem:packing}
The packing number of the Euclidean unit sphere in $\mathbb{R}^d$ with distance $\alpha$ is at least $\alpha^{-d}$.
\end{lemma}
\begin{proof}
Let $\sett{x_1,\dots,x_N}$ be the largest possible packing of the unit sphere, with $N < \alpha^{-d}$. Consider the set of points that are within $\alpha$ of one of the points in the packing:
\begin{align}
\vol{x : \exists i \norm{x - x_i} \leq \alpha} 
&= \vol{\bigcup_i B(x_i, \alpha)} \\
&\leq \sum_i \vol{B(x_i,\alpha)} \\
&= N \alpha^d \vol{B(0,1)} \\
&< \vol{B(0,1)}
\end{align}
Therefore, there exists a point $y \in B(0,1)$ such that $\norm{y - x_i} > \alpha$ for all $i$. The existence of such a point contradicts the assumption that $\sett{x_1,\dots,x_N}$ is the largest possible packing. We conclude that the packing number is at least $\alpha^{-d}$. 
\end{proof}

\begin{theorem}\label{thm:memory-lower-bound}
For any $L, B > 0$ and any $\epsilon \leq \frac{LB}{2}$, any optimization algorithm that is guaranteed to return an $\epsilon$-suboptimal point for any convex, $L$-Lipschitz function with $\norm{x^*} \leq B$ must use at least $d \log \frac{LB}{2\epsilon}$ bits of memory.
\end{theorem}
\begin{proof}
To begin, by Lemma \ref{lem:packing} there exists a packing $\sett{x_1,\dots,x_N}$ of the ball $\sett{x: \norm{x} \leq B}$ of size at least $N \geq \parens{\frac{LB}{2\epsilon}}^d$ such that $\norm{x_i - x_j} > \frac{2\epsilon}{L}$ for all $i \neq j \in [N]$. We will associate a function with each point in the packing, let
\begin{equation}
    f_i(x) = L\norm{x - x_i}
\end{equation}
These functions are convex and $L$-Lipschitz, and their optimizers $x_i$ have norm less than $B$.

Note that any point $x$ which is an approximate minimizer of some $f_i$ must have high function value on all other functions $f_j$. Suppose $f_i(x) \leq \epsilon$, then $\norm{x - x_i} \leq \frac{\epsilon}{L}$. Consequently, for all $j \neq i$, $\norm{x - x_j} = \norm{x - x_i + x_i - x_j} \geq \norm{x_i - x_j} - \norm{x - x_i} > \frac{\epsilon}{L}$, thus $f_j(x) = L\norm{x - x_j} > \epsilon$.

Consider using a memory-bounded optimization algorithm to optimize one of these functions $f_i$. After the algorithm has made all of its first-order oracle accesses, the output function $\out$ must map from the final memory state $\mu_T$ to a solution $\hat{x}$. Suppose the final memory state $\mu_T$ uses $M < \log N$ bits, then there are at most $2^M < N$ outputs that the algorithm might give. However, as we just argued, there exist $N$ functions such that returning an accurate solution for any one of them requires returning an inaccurate solution for all the others. Consequently, any algorithm which can output fewer than $N$ different outputs will fail to optimize at least one of the functions $f_1,\dots,f_N$. 
\end{proof}

\end{document}